\documentclass[lettersize,onecolumn]{IEEEtran}
\usepackage{amsmath,amsfonts}
\usepackage{algorithmic}
\usepackage{algorithm}
\usepackage{array}
\usepackage[caption=false,font=normalsize,labelfont=sf,textfont=sf]{subfig}
\usepackage{textcomp}
\usepackage{stfloats}
\usepackage{bbold}
\usepackage{url}
\usepackage{graphicx}
\usepackage{subcaption}
\usepackage{verbatim}

\usepackage{cite}
\usepackage{algorithm}
\usepackage{algorithmic}

\usepackage{amsfonts,amsmath,amsthm, mathtools}
\usepackage{amssymb}
\usepackage{caption}
\usepackage{verbatim}
\usepackage{mathrsfs}
\usepackage{algorithm}
\usepackage{algorithmic}
\usepackage{comment}
\usepackage{multirow, tabularx, booktabs}
\usepackage[dvipsnames]{xcolor}

\usepackage{hyperref}
\usepackage{dsfont}

\newtheorem{theorem}{Theorem}

\newtheorem{lemma}[theorem]{Lemma}

\newtheorem{definition}[theorem]{Definition}

\usepackage{fullpage}
\usepackage{xspace}

\numberwithin{equation}{section}

\theoremstyle{plain}
\newtheorem{thm}{Theorem}[section]
\newtheorem*{theorem*}{Theorem}

\theoremstyle{definition}
\newtheorem{defn}{Definition}[section]
\newtheorem{rem}{Remark}[section]
\usepackage[margin=1in]{geometry}

\newcommand{\rr}{\mathbb{R}}

\newcommand{\pp}{\mathcal{P}}
\newcommand{\hh}{\mathcal{H}}

\newcommand{\mm}{\mathcal{M}}

\newcommand{\xl}{\mathcal{X}}
\newcommand{\yl}{\mathcal{Y}}
\newcommand{\ec}{\mathcal{E}}
\newcommand{\ee}{\mathbb{E}}

\newcommand{\dd}{\mathbb{D}_c}

\newcommand{\xx}{\mathcal{X}}

\begin{document}

\title{Transformers Are Universally Consistent}

\author{
    \IEEEauthorblockN{Sagar Ghosh\IEEEauthorrefmark{1}, Kushal Bose\IEEEauthorrefmark{1}, and Swagatam Das\IEEEauthorrefmark{1}, } \\
    \IEEEauthorblockA{\IEEEauthorrefmark{1}Indian Statistical Institute
    \\sagarghosh1729@gmail.com}, {kushalbose92@gmail.com} {swagatam.das@isical.ac.in}
    \\
}



\maketitle

\begin{abstract}

Despite their central role in the success of foundational models and large-scale language modeling, the theoretical foundations governing the operation of Transformers remain only partially understood. Contemporary research has largely focused on their representational capacity for language comprehension and their prowess in in-context learning, frequently under idealized assumptions such as linearized attention mechanisms. Initially conceived to model sequence-to-sequence transformations, a fundamental and unresolved question is whether Transformers can robustly perform functional regression over sequences of input tokens. This question assumes heightened importance given the inherently non-Euclidean geometry underlying real-world data distributions. In this work, we establish that Transformers equipped with softmax-based nonlinear attention are uniformly consistent when tasked with executing Ordinary Least Squares (OLS) regression, provided both the inputs and outputs are embedded in hyperbolic space. We derive deterministic upper bounds on the empirical error which, in the asymptotic regime, decay at a provable rate of $\mathcal{O}(t^{-1/2d})$, where $t$ denotes the number of input tokens and $d$ the embedding dimensionality. Notably, our analysis subsumes the Euclidean setting as a special case, recovering analogous convergence guarantees parameterized by the intrinsic dimensionality of the data manifold. These theoretical insights are corroborated through empirical evaluations on real-world datasets involving both continuous and categorical response variables.
\end{abstract}


\section{Introduction}
\IEEEPARstart{T}{he} advent of the self-attention-based Transformer architecture~\cite{vaswani2017attention} has profoundly influenced the domains of natural language processing (NLP), computer vision, and speech recognition. Initially introduced as a sequence-to-sequence model, the Transformer has demonstrated remarkable efficacy in tasks such as machine translation~\cite{wang2019learning}, language modeling~\cite{bouschery2023augmenting}, text generation~\cite{zhang2023survey}, and question answering~\cite{shao2019transformer}. Its capacity to model token-to-token mappings with unprecedented performance has led to widespread adoption, extending to domains such as audio processing~\cite{dong2018speech, gulati2020conformer} and the chemical and biological sciences~\cite{rives2021biological, schwaller2019molecular}.

The standard Transformer architecture consists of two primary components: an encoder and a decoder. The encoder block is composed of a self-attention mechanism and a token-wise feed-forward network, while the decoder includes an additional masked attention mechanism. The encoder receives positional encodings alongside input token embeddings and utilizes the attention mechanism to compute context-sensitive weighted representations of each token. These representations are independently processed by the feed-forward layer. To address vanishing gradients and stabilize training dynamics, residual connections~\cite{gu2019improving} and layer normalization are incorporated within each sublayer. The decoder integrates cross-attention with masked multi-head attention and feed-forward layers, ensuring autoregressive constraints during generation by masking future tokens. We begin with a brief mathematical exposition of the attention and feed-forward mechanisms.

\subsubsection*{\textbf{Attention Mechanism}}

Each attention head in a multi-head attention layer employs trainable matrices—Queries ($Q$), Keys ($K$), and Values ($V$). Given a sequence of $t$ input tokens $\{X_i\}_{i=1}^t \subset \mathbb{R}^d$, the attention mechanism is defined as:
\begin{equation}
    \textrm{Attn}(Q,K,V) := \sigma(QK^\top)V,
\end{equation}
where $\sigma$ denotes the row-wise softmax operation. In practice, the dot product $QK^\top$ is typically scaled by $1/\sqrt{d}$. The resulting matrix $\sigma(QK^\top)$ is often referred to as the attention matrix. The outputs from all heads are concatenated and projected to form the final output of the attention layer.

The Transformer employs three forms of attention: self-attention in the encoder to capture contextual dependencies among input tokens; masked self-attention in the decoder to preserve autoregressive structure; and cross-attention in the decoder to condition on encoder outputs. For comprehensive discussions, we refer readers to~\cite{survey}.

\subsubsection*{\textbf{Feed-Forward Mechanism}}

The token-wise feed-forward network processes each token independently through a fully connected two-layer architecture with ReLU activation. The transformation is expressed as:
\begin{equation}
    \textrm{FF}(X) := U_2 \cdot \mathrm{ReLU}(U_1 \cdot \mathrm{Attn}(X) + b_1\mathbb{1}_t^\top) + b_2\mathbb{1}_t^\top,
\end{equation}
where $U_1, U_2$ are weight matrices, $b_1, b_2$ are bias vectors, and $\mathbb{1}_t$ is the $t$-dimensional vector of ones.

Despite their empirical success, the theoretical underpinnings of Transformer architectures remain largely unexplored. A foundational result in this direction is the universal approximation theorem for Transformers by Yun et al.~\cite{approx_tran}, which demonstrated that even a minimalist Transformer—with two single-dimensional attention heads and a hidden layer of width four—can approximate any sequence-to-sequence function with compact support on $\mathbb{R}^{d \times t}$, provided the positional encodings are incorporated. Their work formalized the notion of contextual embeddings, showing that the attention mechanism can compute token interactions while the feed-forward network approximates permutation-equivariant functions. Injecting positional information effectively lifts this equivariance constraint.

While the universal approximation result provides valuable insights into representational capacity, it leaves open the question of \emph{universal consistency}. In particular, we must precisely define what consistency means in the context of Transformer-based sequence modeling. Given a sequence-to-sequence regression task, the generalization error can be viewed as the expected output conditioned on the input sequence. One may then ask whether the empirical risk converges to the true risk, under appropriate regularization and truncation, as the sample size increases.

The notion of universal consistency has previously been studied for convolutional neural networks~\cite{lin}, employing a framework based on metric entropy, pseudo-dimension, and bounded sample analysis, followed by an appeal to the universal approximation property for 1D convolutional architectures. Inspired by this approach, we develop a parallel framework for Transformers by establishing regularity conditions under which the empirical error converges to the generalization error. 

Beyond the Euclidean case, we introduce a novel variant of the Transformer—termed the Hyperbolic Transformer (HyT)—built upon the Poincar\'e ball model. We prove that HyT exhibits universal consistency under analogous conditions, thereby extending the theory to data residing in non-Euclidean geometries.

\subsubsection*{\textbf{Our Contributions}}

\begin{itemize}
    \item We propose a hyperbolic analogue of the vanilla Transformer (HyT), constructed on the Poincar\'e ball and equipped with Hyperbolic Attention (HypAttn) and Hyperbolic Feed-Forward (HypFF) layers, designed to support theoretical analysis in non-Euclidean settings.
    \item We prove that HyT is universally consistent. Moreover, the consistency of the vanilla Transformer follows as a limiting case when the hyperbolic curvature parameter tends to zero.
    \item Our empirical evaluations demonstrate improved performance of HyT over vanilla Transformers during the pre-training phase across diverse datasets.
\end{itemize}

\section{Related Works}\label{sec:3}
\subsection*{\textbf{Attention Mechanisms and It's Imperatives}}
Following the self-attention mechanism described in \cite{vaswani2017attention}, numerous attempts have been made to devise various types of Attention Modules and their direct implementations to find contextual meaning in NLP tasks. The observation, that trained Transformer produces an Attention Matrix with a considerbale sparsity across data points (due to most of the tokens do not carry any contextual linking with most other tokens), gave rise to the concept of Sparse Attention Mechanism \cite{child2019generating}, for which the components of the Attention Matrix are not stored where there is no contextual relation among tokens. Following that, a number of Sparse Attention Mechanisms have been proposed till date: for example, Band Attention and Global Attention [see Star Transformer \cite{qipeng2019star}] , Inter Global Node Attention [see LongFormer, \cite{beltagy2020longformer}], External Global Node Attention [see BigBird, \cite{zaheer2020big}, \cite{ho2019axial}]. A block based splitting of queries and keys followed by a Sinkhorn normalization procedure to produce a doubly-stochastic assignment based Attention Matrix [see Sparse Sinkhorn Attention \cite{tay2020sparse}] can also facilitate the model to model locality. 

\subsection*{\textbf{Neural Networks and Universal Approximations}} 
One of the most classical results in the theory of Neural Networks is the universal approximation theorems [see \cite{cybenko1989approximation}, \cite{hornik1991approximation}]. They proved that any neural network with one hidden layer with sufficiently long width can arbitrarily approximate continuous functions of compact support. Several other works focused on bounding the depth of the hidden layer [see \cite{lu2017expressive}, \cite{lin2018resnet}]. Hanin and Selke \cite{clark2019does} proved that any neural network can approximate a scalar valued continuous function with hidden layer size $s+1$, provided the input layer has a size of $s$. Moreover, Zhou \cite{zhou2020universality} proved the universal approximation property in the context of a one-dimensional convolutional neural network. In 2019, Yun and others \cite{approx_tran} showed that Transformers having residual connection and softmax activations in their Attention Layer can act as universal approximators of sequence to sequence functions of compact support, provided the input tokens are injected with positional information. Additionally, Sannai \cite{sannai2019universal} showed the universality of permutation equivariant functions using fully connected deep neural networks with ReLU activations.

\section{Preliminaries}\label{sec:4}

This section discusses the preliminaries of Riemannian Manifolds and Hyperbolic Geometry which would underpin the introduction of our proposed framework. 

\vspace{5pt}
\noindent
\textbf{Riemannian manifold, Tangent space, and Geodesics}
An \textit{$n$-dimensional Manifold} $\mathcal{M}$ is a second-countable Hausdorff topological space that locally resembles $\rr^n$ \cite{tu}. For each point $x\in\mathcal{M}$, we can define the \textit{Tangent Space} $T_x(\mm)$ as the first-order linear approximation of $\mm$ at $x$. We call $\mm$ as a \textit{Reimannian Manifold} if there is a collection of metrics $g:=\{g_x:T_x(\mm)\times T_x(\mm)\to\rr , x\in\mm\}$ at every point of $\mm$ \cite{carmo}.  This metric induces a distance function between two points $p,q\in\mm$ joined by a piecewise smooth curve $\gamma:[a,b]\to\mm$ with $\gamma(a)=p, \gamma(b)=q$ and the distance between $p$ and $q$ is defined as $L(\gamma):=\int_{a}^bg_{\gamma(t)}(\gamma^\prime(t),\gamma^\prime(t))^{1/2}dt$. The notion of \textit{Geodesic} between two such points is meant to be that curve $\gamma$ for which $L(\gamma)$ attains the minimum and that $L(\gamma)$ is referred as the \textit{Geodesic Distance} between $p$ and $q$ in that case. Given such a Riemannian Manifold $\mm$ and two linearly independent vectors $u$ and $v$ at $T_x(\mm)$, we define the sectional curvature at $x$ as $k_x(u,v):=\frac{g_x(R(u,v)v,u)}{g_x(u,u)g_x(v,v)-g_x(u,v)^2}$, where $R$ being the Riemannian curvature tensor defined as $R(u,v)w:=\nabla_u\nabla_v w-\nabla_v\nabla_u w-\nabla_{(\nabla_uv-\nabla_vu)}w$ [$\nabla_uv$ is the directional derivative of $v$ in the direction of $u$, which is also known as the \textit{Rimannian Connection} on $\mm$. ] 

We use these notions to define an $n$-dimensional model \textit{Hyperbolic Space} as the connected and complete Riemannian Manifold with a constant negative sectional curvature.  There are various models in use for Hyperbolic Spaces, such as Poincar\'{e} Ball Model, Poincar\'{e} half Space Model, Klein-Beltrami Model, Hyperboloid Model, etc, but the celebrated \textit{Killing-Hopf Theorem} \cite{lang} asserts that for a particular curvature and dimension, all the model hyperbolic spaces are isometric. This allows us to develop our architecture uniquely (without worrying much about performance variations) over a particular model space, where we choose to work with the Poincar\'{e} Ball model for our convenience. Here, we have briefly mentioned the critical algebraic operations on this model required for our purpose.  

\vspace{5pt}
\noindent
\textbf{Poincar\'{e} Ball Model}
For a particular curvature $k(<0) [c=-k]$, an $n-$ dimensional Poincar\'{e} Ball model contains all of its points inside the ball of radius $1/\sqrt{c}$ embedded in $\rr^n$ \cite{lee}. The geodesics in this model are circular arcs perpendicular to the spherical surface of radius $1/\sqrt{c}$. The geodesic distance between two points $p$ and $q$ (where $\|p\|,\|q\|<1/\sqrt{c}$) is defined as
\begin{align}\label{eqn:metric}
    d(p,q):=2\sinh^{-1}\left(\sqrt{2\frac{\|p-q\|^2}{c(\frac{1}{c}-\|p\|^2)(\frac{1}{c}-\|q\|^2)}}\right).
\end{align}

From now on, we will denote $\dd^n$ as the $n-$ dimensional Poincar\'{e} Ball with curvature $-c$.  

\vspace{5pt}
\noindent
\textbf{Gyrovector Space}
The concept of Gyrovector Space, introduced by Abraham A. Ungar [see \cite{ungar}], serves as a framework for studying vector space structures within Hyperbolic Space. This abstraction allows for defining special addition and scalar multiplications based on weakly associative gyrogroups. For a detailed geometric formalism of these operations, Vermeer's work \cite{vermeer} provides an in-depth exploration.

In this context, we will briefly discuss M$\ddot{o}$bius Gyrovector Addition and Mobius Scalar Multiplication on the Poincar'{e} Disc. Due to isometric transformations between hyperbolic spaces of different dimensions, the same additive and multiplicative structures can be obtained for other model hyperbolic spaces (refer \cite{ungar}). Utilizing M$\ddot{o}$bius addition and multiplication is essential when evaluating intrinsic metrics like the Davies-Bouldin Score or Calinski-Harabasz Index to assess the performance of our proposed algorithm.

\begin{enumerate}
    \item \textbf{M$\ddot{o}$bius Addition:} For two points $u$ and $v$ in the Poincar\'{e} Ball, the M$\ddot{o}$bius addition is defined as:
    \begin{align}\label{eqn:mob_add}
        u\oplus_c v:= \frac{(1+2c<u,v>+c\|v\|^2)u+(1-c\|u\|^2)v}{1+2c<u,v>+c^2\|u\|^2\|v\|^2},
    \end{align}
        
     where $c$ is the negative of the curvature of the Poincar\'{e} Ball.

    \item \textbf{M$\ddot{o}$bius Scalar Multiplication:} For $r\in\rr$, $c>0$ and $u$ in the Poincar\'{e} Ball, the scalar multiplication is defined as:
    \begin{align}\label{eqn:mob_mult}
        r\otimes_c u:= \frac{1}{\sqrt{c}} \tanh\left(r \tanh^{-1}(\sqrt{c}\|u\|)\right)\frac{u}{\|u\|}.
    \end{align}

    This addition and scalar multiplication satisfy the axioms about the Gyrovector Group [see \cite{ungar}].
\end{enumerate}

\vspace{5pt}
\noindent
\textbf{Fr\'{e}chet Centroid} For a set of $m$ points $\{x_1,x_2,...,x_m\}\in\dd^n$, we define the Fr\'{e}chet centroid as a generalized notion of the Euclidean Centroid, defined as 
\begin{align}\label{eqn:fc}
    FC(x_1,x_2,...,x_m):=\frac{1}{m}\otimes_c \left(x_1\oplus_c\left(x_2\oplus_c...\left(x_{m-1}\oplus_c x_m\right)\right)\right).
\end{align}

\vspace{5pt}
\noindent
\textbf{Exponential \& Logarithmic Maps}
For any $x\in\dd^n$, the $\exp_x^c:T_x(\dd^n)\subseteq \rr^n\to\dd^n$ translates a point in the tangent space of the Poincar\'{e} Ball and projects it on the Poincar\'{e} Ball along the unit speed geodesic starting from $x\in\dd^n$ in the direction $v\in T_x(\dd^n)$. The Logarithmic map does precisely the opposite, i.e., $\log_x^c:\dd^n\to T_x(\dd^n)\subseteq \rr^n$, projecting a point from the Poincar\'{e} Ball back to the tangent space at $x\in\dd^n$ along the reverse of the geodesic traced by the Exponential Map. Their formulations are explicitly given as follows:
\begin{align}\label{eqn:exp_map}
    \exp_x^c(v):=x\oplus_c\left(\tanh\left(\sqrt{c}\frac{\lambda_x^c\|v\|}{2}\right)\frac{v}{\sqrt{c}\|v\|}\right)
\end{align}
and 
\begin{align}\label{eqn:log_map}
    \log_x^c(y):=\frac{2}{\sqrt{c}\lambda_x^c}\tanh^{-1}\left(\sqrt{c}\|-x\oplus_c y\|\right)\frac{-x\oplus_c y}{\|-x\oplus_c y\|},
\end{align}
for $y\neq x$ and $v\neq 0$ and the Poincar\'{e} conformal factor $\lambda_x^c:=\frac{2}{(1-c\|x\|^2)}$. 

\vspace{5pt}
\noindent
\textbf{Gromov Hyperbolicity} For any metric space $(X,d)$, the \emph{Gromov Product} of two points $x,y$ with respect to a third point $w$ is defined as 
\begin{align*}
    (x,y)_w:=\frac{1}{2}(d(x,w)+d(y,w)-d(x,y))
\end{align*}
and we say $X$ is $\delta-$ hyperbolic iff for any tuple $(x,y,z,w)$ of four points in $X$, we have
\begin{align}\label{eqn:gr1}
    (x,z)_w\geq \min((x,y)_w,(y,z)_w)-\delta.
\end{align}

\section{Proposed Architecture}\label{sec:5}

In this section, we will describe our Hyperbolic Transformer (HyT) Architecture. We assume the input $X$ to the encoder of this Architecture is in $\dd^{d\times t}$, i.e. a matrix of order $d\times t$ whose columns are embedded words or token vectors after getting mapped to the Poincar\'{e} Ball $\dd^d$ through the Exponential Map, where $d$ is the embedding dimension. The encoder block of a HyT Architecture can be thought of as a sequence to sequence function mapping $\dd^{d\times t}\to\rr^{d\times t}$, similar to what is described in \cite{approx_tran}. It will have two components: a Hyperbolic Self-Attention Layer (HypAttn) and a fully Hyperbolic Feed Forward Network (HypFF), while both layers will be equipped with a skip-connection thorough a M$\ddot{o}$bius Addition operation. For such an input $X$ consisting of $t$ tokens, we define the layers of HyT as:
\begin{equation} \label{eqn:attn}
    \textrm{HypAttn}(X):=X\oplus_c \exp_0^c\left[\sum_{j=1}^h U_f^jU_v^j\log_0^c(X)\cdot\sigma\left(\left(U_k^j\log_0^c(X)\right)^tU_q^j\log_0^c(X)\right)\right]
\end{equation}
and 
\begin{equation}\label{eqn:ff}
    \textrm{HypFF}(X):=\log_0^c\left[\textrm{HypAttn}(X)\oplus_c \exp_0^c\left[U_2\cdot \textrm{ReLU}\left(U_1\cdot \log_0^c\left(\textrm{HypAttn}(X)\right)\right)+l_1\mathbb{1}_m^t\right]\oplus_c \exp_0^c(l_2\mathbb{1}_m^t)\right],
\end{equation}
where $\{U_f^j\}_{j=1}^h\in\rr^{d\times s}$ are the concatenating matrices. For each head index $j$, the query, key and value matrices respectively are $U_v^j,U_k^j,U_q^j\in\rr^{s\times d}$. $U_2\in\rr^{d\times r}$, $U_1\in\rr^{r\times d}, l_1\in\rr^r,l_2\in\rr^d$. There are three parameters of the HyT Architecture: the number of heads $h$, head size $s$ are the two parameters of the HypAttn layer and the hidden layer size $r$ of the HypFF layer.

We represent the class of sequence to sequence functions from $\dd^{d\times t}$ to $\rr^{d\times t}$ that can be expressed by our HyT Architecture as
\begin{equation}
    \mathcal{T}_{\mathcal{H}}^{h,s,r}:=\{f:\dd^{d\times t}\to\rr^{d\times t} | f \text{ is a finite composition of HyT blocks} \hspace{1ex} t_{\mathcal{H}}^{h,s,r}\},
\end{equation}
where $t_{\mathcal{H}}^{h,s,r}:\dd^{d\times t}\to\rr^{d\times t}$ denotes a HyT block consisting of a HypAttn layer with parameters $h$ and $s$ and a HypFF layer with parameter $r$. 

To this end, we will add a learnable hyperbolic positional encoding to our HyT Architecture. The class of functions which can be expressed by our HyT Architecture attached with a hyperbolic positional encoding is represented as:
\begin{equation}\label{eqn:trans_seq}
    \mathcal{T}_{\pp,\mathcal{H}}^{h,s,r}:=\{f_\pp(X):=f(X\oplus_c E) | f\in\mathcal{T}_\hh^{h,s,r} \hspace{1ex} \text{and} \hspace{1ex} E\in\dd^{d\times t}\}. 
\end{equation}

\section{Theoretical Analyses}\label{sec:6}

We will assume that the input tokens to the HyT are taken from the dataset $\mathcal{D}:=\{(X_i,Y_i)\}_{i=1}^n$, where $X_i$s are the inputs and $Y_i$s are the target output from our HypFF layer. Each $X_i$ consists of $t$ many tokens embedded in $\dd^d$, i.e. $X_i:=\{X_{i,1},X_{i,2},...,X_{i,t}\}\subseteq \dd^{d\times t}$. We also assume that the $X_i$s are drawn independently and from an unknown Borel Probability Measure $\rho$ on $\mathcal{Z}:=\xx\times \mathcal{Y}$. Throughout the paper we consider $\xx$ to be compact (this will be required for employing the approximation property). We aim to learn a sequence to sequence functions $f:\dd^{d \times t}\to\rr^{d\times t}$ (to be approximated by our HyT Architecture) so as to minimize the following $L^2$ generalization error:
\begin{equation}
    \mathcal{E}(f):=\int_{\xx\times\mathcal{Y}}\|f(x)-\log_0^c(y)\|^2 d\rho.
\end{equation}

\begin{lemma}
    The \textit{Hyperbolic Regression Function (HRF)} $f_\rho(x):=\int_{\mathcal{Y}}\log_0^c(y)d\rho(y|x)$, defined through the conditional distribution $\rho(\cdot|x)$ of $\rho$ at $x\in\mathcal{X}$ minimizes the HGE.
\end{lemma}
\begin{proof}
     The $L^2$ generalization error can be written in terms of conditional expectation in the following way:
    \begin{align*}
        \ec(f)&=\int_{\mathcal{Z}}(f(x)-\log_0^c(y))^2d\rho\\
        &=\mathbb{E}_{\mathcal{X},\mathcal{Y}}[f(\mathcal{X})- \log_0^c(\mathcal{Y})]^2
    \end{align*}
    Now for any function $g:\mathcal{X}\to\rr^1$, we write
    \begin{align*}
        \ec(g) =& \ee_{\xl}\left[\ee_{\yl|\xl}\left[ \left(g(\xl)-\ee[\log_0^c(\yl)|\xl]+\ee[\log_0^c(\yl)|\xl]-\log_0^c(\yl)\right)^2|\xl\right] \right]    \\
        = & \ee_{\xl}\left[\ee_{\yl|\xl}\left[\left(g(\xl)-\ee[\log_0^c(\yl)|\xl]\right)^2|\xl\right]\right]+\ee_{\xl}\left[\ee_{\yl|\xl}\left[\left(\ee[\log_0^c(\yl)|\xl]-\log_0^c(\yl)|\xl\right)^2|\xl\right]\right]\\
        & +2\ee_{\xl}\left[\ee_{\yl|\xl}\left[\left(g(\xl)-\ee[\log_0^c(\yl)|\xl]\right)\left(\ee[\log_0^c(\yl)|\xl]-\log_0^c(\yl)\right)|\xl\right]\right]
    \end{align*}
    The cross term in the last expression is $0$, since 
    \begin{align*}
        \ee_{\xl}\left[\ee_{\yl|\xl}\left[\left(\ee[\log_0^c(\yl)|\xl]-\log_0^c(\yl)\right)\right]\right]=0.
    \end{align*}
    Therefore, the expression for HGE is reduced to
    \begin{align*}
        \ec(g) = \ee_{\xl}\left[\ee_{\yl|\xl}\left[\left(g(\xl)-\ee[\log_0^c(\yl)|\xl]\right)^2|\xl\right]\right]+\ee_{\xl}\left[\ee_{\yl|\xl}\left[\left(\ee[\log_0^c(\yl)|\xl]-\log_0^c(\yl)|\xl\right)^2|\xl\right]\right],
    \end{align*}
    which attains minimum when $g(x)=\ee\left[\log_0^c(\yl)|x\right]$ for each $x\in\xl$. Alternately we write for each $x\in\xl$
    \begin{align*}
        g(x)=\int_{\yl}\log_0^c(y)d\rho(y|x). 
    \end{align*}
\end{proof}
 
Now from Lemma 8 of \cite{ghosh} we conclude that the following Regression Function which is defined through the conditional distribution of $Y\in\mathcal{Y}$ given $X\in\xx$ 
\begin{equation}\label{eqn:emp_risk_min}
    f_\rho(x):=\left(\int_{\mathcal{Y}}\log_0^c\left(y\right)d\rho(y|x)\right)
\end{equation}
 minimizes the $L^2$ generalization error. 

We will construct the sequence-to-sequence function $f:\dd^{d\times t}\to\rr^{d\times t}$ through minimizing the emperical risk:
\begin{align*}
    f_{D}^{h,s,r}:=\arg \min_{f\in\mathcal{T}_{\mathcal{P},\mathcal{H}}^{h,s,r}} \mathcal{E}_D(f),
\end{align*}
where $\mathcal{E}_D(f):=\frac{1}{t}\sum_{i=1}^t\|f(X_i)-\log_0^c(Y_i)\|^2_{L^2}$ is the emperical risk, where $\mathcal{T}_{\mathcal{P},\mathcal{H}}^{h,s,r}$ is defined by Equation \ref{eqn:trans_seq}. 

If any learning model or architecture inherits the property that if the sample size $t\to\infty$, then the constructed estimator via the emperical risk minimization (representable by that model) converges to the true estimate of the output. This property is known as strong universal consistency \cite{gyor}, which is formally defined as follows:

\begin{defn}
     A sequence of \textit{Regression Estimators (HRE)} $(\{f_m\}_{m=1}^\infty)$ built through Emperical Risk Minimization is said to be strongly universally consistent if it satisfies the condition:
    \begin{align*}
        \lim_{t \to \infty} \mathcal{E}(f_t) - \mathcal{E}(f_\rho) = 0
    \end{align*}
     almost surely, for every Borel probability distribution $\lambda$ such that $\log_0^c(\mathcal{Y})\in L^2(\rho_{(\mathcal{Y}|x)})$.     
\end{defn}

We will now state the main theorem of this work, which will eventually lead us to prove the universal statistical consistency of our HyT Architecture. While stating the Theorem, we will keep the form of universal approximation property defined in Theorem 3 \cite{approx_tran} and will put $h=2,s=1$ and $r=4$, although we will state the universal approximaion property in the hyperbolic set-up a little later. Considering this fact, here goes the statement of our main Theorem:

\begin{thm} \label{thm:main}
    Suppose $\theta\in(0,1/2d)$ is arbitrary. Then for a number of input tokens $t$, any fixed input embedding dimension $d$ and fixed vocabulary size $v$, if the following conditions hold as $t\to\infty$:
    \begin{enumerate}
        \item $M=M_t\to\frac{1}{\sqrt{c}}$ .

        \item $t^{-\theta}M_t^2\left[1+\frac{1}{M_t\sqrt{c}}\tanh^{-1}(M_t\sqrt{c})\right]^2\to 0$. 
        \item 
      
            $d\frac{P\log(Q)}{t^{1-2\theta d}}\to 0$,
        
        where \begin{align*}
        P&:= \left(\frac{1}{\sqrt{c}}\tanh^{-1}(M_t\sqrt{c})\right)^4 d\log (d)\\
        Q&:= dt\left(\frac{1}{\sqrt{c}}\tanh^{-1}(M_t\sqrt{c})t^\theta\right)^{dt}\log\left(dt\left(\frac{1}{\sqrt{c}}\tanh^{-1}(M_t\sqrt{c})t^\theta\right)^{dt}\right),
    \end{align*}   
    \end{enumerate}
    
    Then $\pi_{M}f_D^{2,1,4}$ is strongly universally consistent, where $\pi_w(u):=\max\{w,|u|\}$ is the well known truncation operator.
\end{thm} 

\begin{rem}
    The universal consistency of the conventional Euclidean transformers is established once we put $\lim c\to 0$ in Theorem \ref{thm:main}. 
\end{rem}

\begin{rem}\label{main_thm_conditions}(\textbf{Interpretations of the constraints in Theorem \ref{thm:main}}) The complex expressions associated with the three conditions in Theorem \ref{thm:main} are not arbitrary; rather, they are rigorously justifiable. The first condition, involving truncation parameters, reflects a progressive placement of input tokens closer to the boundary of the Poincar\'{e} Ball (with radius $1/\sqrt{c}$). The second condition ensures a sufficient number of input tokens relative to their distance from the disc’s center, echoing principles from concentration bounds, by requiring that the growth in token count outpaces the growth in their distance. The third condition involves expressions $P$ and $Q$, which correspond respectively to the pseudo-dimension and the metric entropy bounds of the class $\mathcal{T}_{\pp,\mathcal{H}}^{2,1,4}$. The imposed limit ensures that the combined growth of these complexity measures is dominated by the number of input tokens, a necessary constraint to preserve effective training dynamics.
\end{rem}

\subsection{\textbf{Capacity Estimates for a Class of Functions Represented by HyT}}
Let \(\rho\) be a probability measure on \(\mathcal{X}\), and let \(f: \mathcal{X} \to \mathbb{D}^K\) be a measurable function. We define the \(L^p(\rho)\)-norm of \(f\) as
\[
\|f\|_{L^p(\rho)} := \left[ \int_{\mathcal{X}} \left\| \log_0^c(f(x)) \right\|^p \, d\rho(x) \right]^{1/p},
\]
where, for a vector \(v = \{v_1, v_2, \dots, v_n\} \in \mathbb{R}^n\), we define \(\|v\|^p := \sum_{i=1}^n |v_i|^p\). The space \(L^p(\rho)\) consists of all functions \(f: \mathcal{X} \to \mathbb{D}^K\) such that \(\|f\|_{L^p(\rho)} < \infty\).

Given a function class \(T \subseteq L^p(\rho)\), the covering number \(\mathcal{N}(\epsilon, T, \|\cdot\|_{L^p(\rho)})\) is defined as the smallest integer \(n_0 \in \mathbb{N}\) such that \(T\) can be covered by \(n_0\) balls of radius \(\epsilon\) in the \(\|\cdot\|_{L^p(\rho)}\)-norm.  
Letting \(v_1^m := \{v_1, v_2, \dots, v_m\} \in \mathcal{X}^m\), and denoting by \(\rho_m\) the empirical measure supported on \(v_1^m\), we define the empirical covering number \(\mathcal{N}_p(\epsilon, T, v_1^m) := \mathcal{N}(\epsilon, T, \|\cdot\|_{L^p(\rho_m)})\).

We further define the \(\epsilon\)-packing number \(\mathcal{M}(\epsilon, T, \|\cdot\|_{L^p(\rho)})\) as the maximal cardinality \(n_p \in \mathbb{N}\) of any subset \(\{f_1, f_2, \dots, f_{n_p}\} \subseteq T\) such that
\[
\left\| \log_0^c(f_j) - \log_0^c(f_k) \right\|_{L^p(\rho)} \geq \epsilon \quad \text{for all } 1 \leq j < k \leq n_p.
\]
The empirical packing number is analogously defined as \(\mathcal{M}_p(\epsilon, T, v_1^m) := \mathcal{M}(\epsilon, T, \|\cdot\|_{L^p(\rho_m)})\).

The following classical lemma establishes a relationship between \(\epsilon\)-packing and \(\epsilon\)-covering numbers.

\begin{lemma}\label{lem:1}
    For a class of functions $T$ on $\xx$ equipped with a probability measure $\rho$ we have for $p\geq 1$ and $\epsilon>0$
    \begin{align*}
        \mathcal{M}(2\epsilon,T,\|\cdot\|_{L^p(\rho)})\leq \mathcal{N}(\epsilon,T,\|\cdot\|_{L^p(\rho)})\leq \mathcal{M}(\epsilon,T,\|\cdot\|_{L^p(\rho)}).
    \end{align*}
    When we have $v_1^m\in\xx^m$, we have in particular
    \begin{align*}
        \mathcal{M}_p(2\epsilon,T,v_1^m)\leq \mathcal{N}_p(\epsilon,T,v_1^m)\leq \mathcal{M}_p(\epsilon,T,v_1^m).
    \end{align*}
\end{lemma}

\begin{definition}\label{defn:4.2}
Let \( T \subseteq L^p(\rho) \) be a class of functions taking values in the Poincaré Ball \( \mathbb{D}^K \), where \( \rho \) is a probability measure on the input space \( \mathcal{X} \). Assuming the standard notion of pseudo-dimension for real-valued function classes, we extend the concept to this hyperbolic setting as follows.

We define the \textbf{pseudo-dimension} of \( T \), denoted by \( \mathrm{pdim}(T) \), as the largest integer \( \ell \in \mathbb{N} \) for which there exists a tuple
\[
(x_1, x_2, \dots, x_\ell, \eta_1^1, \dots, \eta_\ell^1, \dots, \eta_1^K, \dots, \eta_\ell^K) \in \mathcal{X}^\ell \times \mathbb{R}^{K\ell}
\]
such that, for every binary labeling \( (c_1, c_2, \dots, c_\ell) \in \{0,1\}^\ell \), there exists a function \( f \in T \) satisfying the condition:
\[
\log_0^c(f(x_i)) > \eta_i := \begin{bmatrix} \eta_1^i \\ \eta_2^i \\ \vdots \\ \eta_K^i \end{bmatrix} \quad \text{if and only if} \quad c_i = 1, \quad \text{for all } i \in \{1, \dots, \ell\}.
\]
\end{definition}

Building upon the previously defined notion of pseudo-dimension for function classes mapping into the Poincaré Ball, our objective is to establish an upper bound on the corresponding \( \epsilon \)-packing number. To achieve this, we adopt the foundational methodology outlined in \cite{lin}, while adapting it to the hyperbolic setting. In particular, our analysis necessitates a suitable modification of Lemma 4 and Theorem 6 from \cite{huss}, to accommodate the non-Euclidean geometry inherent to the Poincaré Ball.  

We begin by presenting a hyperbolic analogue of Lemma 4 from \cite{huss}, together with its proof.
 
\begin{lemma}\label{lemma:4.3}
    Suppose we have a family of functions $F$ from a set $\xx$ (equipped with a probability measure $\rho$) to $\dd^K$ with $\log_0^c(f)\in [-M,M]^K$ for all $f\in F$. Let $R:=\{r_1,r_2,...,r_m\}$ be a random vector in $[-M,M]^{mK}$, where $r_i:=\{r_i^1,r_i^2,...,r_i^K\}$  and each $r_i^j$ is drawn at random from the uniform distribution on $[-M,M]$ for $1\leq i\leq m$ and $1\leq j\leq K$. Let $v:=\{v_1,...,v_m\}$ be a random vector in $\xx^m$, where each $v_i$ is drawn independently at random following $\rho$. Then for all $\epsilon>0$,
    \begin{align*}
        \mathbb{E}\left(|\textrm{sign}((\log_0^c(F))|_{v}-r)|\right)\geq \mathcal{M}\left(\epsilon,F,\|\cdot\|_{L^p(\rho)}\right)\left(1-\mathcal{M}(\epsilon,F,\|\cdot\|_{L^p(\rho)})e^{-m\left(\frac{\epsilon}{2M}\right)^K}\right),
    \end{align*}
    where 
    \begin{align*}
        \log_0^c(F)|_{v}:=\{\log_0^c(f)_{v}:f\in F\}:=\{(\log_0^c(f(v_1)),\log_0^c(f(v_2)),...,\log_0^c(f(v_m))):f\in F\}. 
    \end{align*}
\end{lemma}
\begin{proof}
    We follow a similar technique as it is done in \cite{huss}. We will take $G$ to be an $\epsilon-$ separated (w.r.t. the $L^p(\rho)$ norm) subset of $F$ with $|G|=\mathcal{M}(\epsilon,F,\|\cdot\|_{L^p(\rho)})$. Now 
   \begin{align*}
        \mathbb{E}\left(|\textrm{sign}((\log_0^c(F))|_{v}-r)|\right) & \geq \mathbb{E}\left(|\textrm{sign}((\log_0^c(G))|_{v}-r)|\right)\\
        &\geq \mathbb{E}(|\{f\in G: \textrm{sign}(\log_0^c(f)|_{v}-r)\neq \textrm{sign}(\log_0^c(g)|_{v}-r) \hspace{1ex} \text{for all }\hspace{1ex} g\in G,g\neq f\}|))  \\
        & = \sum_{f\in G} \mathbb{P}(\textrm{sign}(\log_0^c(f)|_{v}-r)\neq \textrm{sign}(\log_0^c(g)|_{v}-r) \hspace{1ex} \text{for all} \hspace{1ex} g\in G, g\neq f)\\
        &= \sum_{f\in G} (1-\mathbb{P}(\exists g\in G, g\neq f: \textrm{sign}(\log_0^c(f)|_{v}-r)= \textrm{sign}(\log_0^c(g)|_{v}-r)))\\
        &\geq \sum_{f\in G} (1-|G|\max_{g\in G, g\neq f} \mathbb{P}(\textrm{sign}(\log_0^c(f)|_{v}-r)= \textrm{sign}(\log_0^c(g)|_{v}-r))).
   \end{align*}
   Now if $f$ and $g$ are $\epsilon-$ separated in $G$ and $\log_0^c(f),\log_0^c(g)\in [-M,M]^K$ and if $r_i$ is drawn at random from the uniform distribution on $[-M,M]^K$, the probability that $sign(\log_0^c(f)|_{v_i}-r_i)\neq sign(\log_0^c(g)|_{v_i}-r_i)$ will be at least $(\epsilon/{2M})^K$ for $1\leq i\leq m$. Hence, we get
   \begin{align*}
       \mathbb{P}(\textrm{sign}(\log_0^c(f)|_{v}-r)=\textrm{sign}(\log_0^c(g)|_{v}-r))\leq \left(1-\left(\frac{\epsilon}{2M}\right)^K\right)\leq e^{-m\left(\frac{\epsilon}{2M}\right)^K}.
   \end{align*}
   Putting everything together and by noting that $|G|=\mathcal{M}(\epsilon, F, \|\cdot\|_{L^p(\rho)})$, we get
   \begin{align*}
       \mathbb{E}\left(|\textrm{sign}((\log_0^c(F))|_{v}-r)|\right)\geq \mathcal{M}\left(\epsilon,F,\|\cdot\|_{L^p(\rho)}\right)\left(1-\mathcal{M}(\epsilon,F,\|\cdot\|_{L^p(\rho)})e^{-m\left(\frac{\epsilon}{2M}\right)^K}\right).
   \end{align*}
  
\end{proof}

Having a modified version of Lemma 4, we can now state and prove a similar version of Theorem 6 in \cite{huss}. 
\begin{thm}\label{thm:4.1}
    For such a family $F$ as mentioned above with $pdim(F)=d$, we have for any $0<\epsilon\leq 2M$
    \begin{align*}
        \mathcal{M}(\epsilon,F,\|\cdot\|_{L^p(\rho)}) < 2\left(2eK\left(\frac{2M}{\epsilon}\right)^K\log\left(2eK\left(\frac{2M}{\epsilon}\right)^K\right)\right)^d.
    \end{align*}
\end{thm}
\begin{proof}
    By Sauer's Lemma, we can write
    \begin{align*}
        |\textrm{sign}(\log_0^c(F)|_{v}-r)|\leq \left(\frac{emK}{d}\right)^d,
    \end{align*}
    for all $m\geq d$ and $r\in [-M,M]^{mK}$. By writing $|G|=\mathcal{M}(\epsilon,F,\|\cdot\|_{L^p(\rho)})$, we get for all $m\geq d$,
    \begin{align*}
        \left(\frac{emK}{d}\right)^d\geq |G|\left(1-|G|e^{-m\left(\frac{\epsilon}{2M}\right)^K}\right).
    \end{align*}
    It is easy to verify the claim of Theorem \ref{thm:4.1} if $\left(\frac{2M}{\epsilon}\right)^K\log(2|G|)<d$ simply by noting that $\epsilon\leq 2M$. So, to prove our claim, we can assume that $\left(\frac{2M}{\epsilon}\right)^K\log(2|G|)\geq d$. Hence, $m\geq \left(\frac{2M}{\epsilon}\right)^K\log(2|G|)\geq d$. Simplifying the first two expressions of the last inequality, we get
    \begin{align*}
        \left(1-|G|e^{-m\left(\frac{\epsilon}{2M}\right)^K}\right)\geq \frac{1}{2}. 
    \end{align*}
    Hence, using Sauer's Lemma, we have
    \begin{align*}
        \left(\frac{emK}{d}\right)^d \geq \frac{1}{2}|G|.
    \end{align*}
    This is true for all $m\geq \left(\frac{2M}{\epsilon}\right)^K\log(2|G|)\geq d$. In particular, this is true for $m=\left(\frac{2M}{\epsilon}\right)^K\log(2|G|)$. A little more computations will prove that
    \begin{align*}
        |G| < 2\left(2eK\left(\frac{2M}{\epsilon}\right)^K\log\left(2eK\left(\frac{2M}{\epsilon}\right)^K\right)\right)^d,
    \end{align*}
    completing the proof of Theorem \ref{thm:4.1}. 
\end{proof}

\begin{rem}
    Note that for $K=1$ and when the function class $F$ itself is real-valued, Lemma \ref{lemma:4.3} and Theorem \ref{thm:4.1} are reduced to respectively Lemma 4 and Theorem 6  in \cite{huss}. 
\end{rem}

Now we calculate the number of trainable parameters in a single functional block of $\mathcal{T}_{\mathcal{P},\mathcal{H}}^{h,s,r}$, which consists of a $\textrm{HypAttn}$ \ref{eqn:attn} layer with parameters $h$ and $s$ and a $\textrm{HypFF}$ \ref{eqn:ff} layer with parameter $r$. 

For each attention head, the query, key, value, and the concatenating matrices have $s\times d$ parameters each, making a total of $4sd$ parameters. For $h$ attention heads, this leads to a total of $4sdh$ trainable parameters for the $\textrm{HypAttn}$ layer. Similarly the $\textrm{HypFF}$ layer with a hidden layer of size $r$ has $2(d\times r)+d+r$ many trainable parameters. Finally the tunable Hyperbolic Input Embedding matrix adds another $d\times v$ parameters, $v$ being fixed the size of the vocabulary. Putting all together, one single $t\in\mathcal{T}_{\mathcal{P},\mathcal{H}}^{h,s,r}$ has
\begin{align*}
    n_{param}=4sdh+2dr+d+r+dv
\end{align*}
many trainable parameters. 

In the $\textrm{HypAttn}$ layer, in each attention head, there are $d$ many input neurons and $s$ many output neurons for each of query, key, and value matrix computations. Furthermore, the concatenating matrix demands another set of $d$ output neurons per head. Therefore the $\textrm{HypAttn}$ layer requires a total of $h[3(d+s)+d]$ neurons. Similarly the $\textrm{HypFF}$ layer has $d+r+d$ neurons [$d$ input neurons, $r$ hidden layer neurons, and $d$ output neurons]. Therefore, the total number of neurons or computational units required for a single block $t\in \mathcal{T}_{\mathcal{P},\mathcal{H}}^{h,s,r}$ in our HyT Architecture is
\begin{align*}
    n_{neuron}=h[3(d+s)+d]+2d+r.
\end{align*}

Now we will provide an upper bound on the Pseudo-Dimension of a truncated class of functions in $\mathcal{T}_{\mathcal{P},\mathcal{H}}^{h,s,r}$ following \cite{lin}, which gets us to the next lemma.

\begin{lemma} \label{lem:5}
    An upper bound for the Pseudo-Dimension of $\left(\pi_M(\mathcal{T}_{\mathcal{P},\mathcal{H}}^{h,s,r})\right)$ can be given as:
    \begin{align}
        \mathcal{P}_{dim}\left(\pi_M(\mathcal{T}_{\mathcal{P},\mathcal{H}}^{h,s,r})\right) \leq c n_{param}\log_2\left(n_{neuron}\right),
    \end{align}
    where $\left(\pi_M(\mathcal{T}_{\mathcal{P},\mathcal{H}}^{h,s,r})\right):=\{f\in \mathcal{T}_{\mathcal{P},\mathcal{H}}^{h,s,r}: \log_0^c(f)\in[-M,M]^{d\times t}\}$.
\end{lemma}
\begin{proof}
    By combining [Theorem 7, \cite{bart}] and [Theorem 14.1, \cite{anthony}] we get that 
    \begin{align*}
        \mathcal{P}_{dim}\left(\mathcal{T}_{\mathcal{P},\mathcal{H}}^{h,s,r}\right) \leq c^\prime n_{param}\log_2\left(n_{neuron}\right)
    \end{align*}
    for some constant $c^\prime >0$. But $\pi_M\left(\mathcal{T}_{\mathcal{P},\mathcal{H}}^{h,s,r}\right)\subseteq \mathcal{T}_{\mathcal{P},\mathcal{H}}^{h,s,r}$. Hence
    \begin{align*}
        \mathcal{P}_{dim}\left(\pi_M\left(\mathcal{T}_{\mathcal{P},\mathcal{H}}^{h,s,r}\right)\right) \leq \mathcal{P}_{dim}\left(\mathcal{T}_{\mathcal{P},\mathcal{H}}^{h,s,r}\right),
    \end{align*}
    completing the proof of Lemma \ref{lem:5}.
\end{proof}

\begin{rem} \label{rem:5.2}
Borrowing the same configuration from \cite{approx_tran}, if we set $h=2, s=1$, and $r=4$, the total number of trainable parameters $n_{param}$ and neurons $n_{neuron}$ become to $l_1d$ and $l_2d$ respectively, where $l_1, l_2>0$ are fixed constants.
\end{rem}

Next, we will give an upper bound on the covering number of $\pi_M\left(\mathcal{T}_{\mathcal{P},\mathcal{H}}^{h,s,r}\right)$, which will complete our finding on the desired capacity estimate of any function in $\pi_M\left(\mathcal{T}_{\mathcal{P},\mathcal{H}}^{h,s,r}\right)$ w.r.t. the $L^2$ norm. 

\begin{lemma}
    There exists an absolute constant $c$ such that for any $0\leq \epsilon\leq M$, 
     \begin{align*}
        \log \sup_{X_1^t\in\mathcal{X}^t}\mathcal{N}\left(\epsilon, \pi_M\left(\mathcal{T}_{\mathcal{P},\mathcal{H}}^{h,s,r}\right), X_1^t \right)\leq & c d n_{param}\log\left(n_{neuron}\right) \times \\
        & \log\left[dt\left(\frac{2M}{\epsilon}\right)^{dt}\log\left(dt\left(\frac{2M}{\epsilon}\right)^{dt}\right)\right].
    \end{align*}
\end{lemma}
\begin{proof}
    For a sequence of $t$ tokens $\{X_i\}_{i=1}^t\in\dd^d$, we can write from Lemma \ref{lem:1} 
    \begin{align}\label{eqn:9.1}
        \mathcal{N}\left(\epsilon, \pi_M\left(\mathcal{T}_{\mathcal{P},\mathcal{H}}^{h,s,r}\right), X_1^t \right)\leq \mathcal{M}\left(\epsilon, \pi_M\left(\mathcal{T}_{\mathcal{P},\mathcal{H}}^{h,s,r}\right), X_1^t\right). 
    \end{align}
    By writing $K=d\times t$, Theorem \ref{thm:4.1} enables us to write
    \begin{align}\label{eqn:9.2}
        \log\left(\mathcal{M}\left(\epsilon, \pi_M\left(\mathcal{T}_{\mathcal{P},\mathcal{H}}^{h,s,r}\right), X_1^m\right)\right)\leq c_1 \mathcal{P}_{dim}\left(\pi_M\left(\mathcal{T}_{\mathcal{P},\mathcal{H}}^{h,s,r}\right)\right)\log\left[dt\left(\frac{2M}{\epsilon}\right)^{dt}\log\left(dt\left(\frac{2M}{\epsilon}\right)^{dt}\right)\right].
    \end{align}
    But from Lemma \ref{lem:5} we have
    \begin{align}\label{eqn:9.3}
        \mathcal{P}_{dim}\left(\pi_M(\mathcal{T}_{\mathcal{P},\mathcal{H}}^{h,s,r})\right) \leq c (4sdh+2dr+d+r+dv) \log\left(h(3(d+s)+d)+2d+r\right). 
    \end{align}
    Finally combining Inequalities \ref{eqn:9.1}, \ref{eqn:9.2} \ref{eqn:9.3}, and Remark \ref{rem:5.2} we get
    \begin{align*}
        \log \sup_{X_1^t\in\mathcal{X}^t}\mathcal{N}\left(\epsilon, \pi_M\left(\mathcal{T}_{\mathcal{P},\mathcal{H}}^{h,s,r}\right), X_1^t \right)\leq & C_1 d n_{param} \log\left(n_{neuron}\right) \times \\
        & \log\left[dt\left(\frac{2M}{\epsilon}\right)^{dt}\log\left(dt\left(\frac{2M}{\epsilon}\right)^{dt}\right)\right].
    \end{align*}
\end{proof}

\subsection{\textbf{Finite Sample Error Bound of the Class HyT}}
Our primary objective is to demonstrate that the empirical error sequences corresponding to finite and bounded samples converge to the universal error constructed through empirical risk minimization, as described in \ref{eqn:emp_risk_min}. To do that, we will provide an upper bound on the difference between the truncated empirical errors and the truncated form of generalization error. Similarly to \cite{lin} and \cite{ghosh}, we proceed by defining the following: 

For any $f\in\mathcal{T}_{\mathcal{P},\mathcal{H}}^{h,s,r}$, we define
\begin{align}
    \mathcal{E}_{\pi_M}(f):=\int_{\xx\times \mathcal{Y}}\|f(x)-\log_0^c(y_M)\|^2d\rho.
\end{align}
 and 
 \begin{align}
     \mathcal{E}_{\pi_{M},D}(f):=\frac{1}{t}\sum_{i=1}^t\|f(x_i)-\log_0^c(y_{i,M})\|^2,
 \end{align}
 where $\pi_M(z):=z_M:=\min\{M,|z|\}\cdot\textrm{sign}(z)$, for any $z\in\rr$ and for $z$ in higher dimensional spaces, the truncation is done component-wise. 

 As in \cite{ghosh}, we have to use a version of concentration inequality, which is a slightly more generalized version of Lemma 17 \cite{ghosh}, which is in turn a version of Lemma 11.4 \cite{gyor}. 

 \begin{lemma}\label{lem:5.2.4}
     We consider $\|y\|\leq B$ and $B\geq \frac{1}{\sqrt{c}}$. For a class of functions $f\in T$ from $\dd^{d\times t}\to\rr^{d\times t}$ satisfying $\|f(x)\|\leq B$, we have 
     \begin{align}
         & \mathbb{P}[\exists f\in \mathcal{T} : \epsilon(f)-\epsilon(f_\rho)-(\epsilon_D(f)-\epsilon_D(f_\rho))\geq \epsilon(\alpha+\beta+\epsilon(f)-\epsilon(f_\rho))] \\
        &\leq 14 \sup_{x_1^m\in \xx^m} \mathcal{N}_1\left(\frac{\beta\epsilon}{20B},\mathcal{T},x_1^m\right)\exp\left(-\frac{\epsilon^2(1-\epsilon)\alpha m}{214(1+\epsilon)B^4}\right), 
     \end{align}
     for all $m\geq 1$ and $\alpha, \beta>0$ and $\epsilon\in(0,1/2)$. 
 \end{lemma}

\begin{proof}
    
 Because $(\pi_Mf_{D}^{h,s,r}), y_M,y_{i,M}\in 
\left[-\frac{1}{\sqrt{c}}\tanh^{-1}(M\sqrt{c}),\frac{1}{\sqrt{c}}\tanh^{-1}(M\sqrt{c})\right]^{dt}$,
    \begin{equation*}
        \left\|\mathcal{E}_{\pi_M}\left(\pi_M f_{D}^{h,s,r}\right)-\mathcal{E}_{\pi_{M},D}\left(\pi_M f_D^{h,s,r}\right)\right\|\leq 8d\left(\frac{1}{\sqrt{c}}\tanh^{-1}(M\sqrt{c})\right)^2.
    \end{equation*}
    
   By putting $\alpha=\beta=1$ and $\epsilon=t^{-\theta}$ in Lemma \ref{lem:5.2.4} we get,
   \begin{equation*}
       \left(\mathcal{E}_{\pi_M}\left(\pi_M f_{D}^{h,s,r}\right)-\mathcal{E}_{\pi_M}(f_\rho)\right)-\left(\mathcal{E}_{\pi_{M},D}(\pi_M f_{D}^{h,s,r})-\mathcal{E}_{\pi_{M},D}(\pi_{M}f_\rho)\right)\leq 8d\left(\frac{1}{\sqrt{c}}\tanh^{-1}(M\sqrt{c})\right)^2t^{-\theta}
   \end{equation*}
   holds with probability at most 
   \begin{equation*}
       14\sup_{x_1^t} \mathcal{N}_1\left(\left(\frac{t^{-\theta}}{2M}\right)^K\frac{1}{20}, \pi_B\mathcal{T},x_1^t\right)\exp\left(-\frac{t^{1-2d\theta}}{214\times (3M^4)^d}\right).
   \end{equation*}
   We will now use Lemma \ref{lem:5.2.4} to derive a complete bound on the Covering Number of the space $\pi_M(\mathcal{G}_L)$, given as follows:
   \begin{align*}
       & \sup_{x_1^t} \mathcal{N}_1\left(\left(\frac{t^{-\theta}}{2M}\right)^d\frac{1}{20}, \pi_M\mathcal{T},x_1^t\right)\exp\left(-\frac{t^{1-2d\theta}}{214\times (3M^4)^d}\right) \\
       & \leq \exp\left[A \log(B)\right] \times  \exp\left(-\frac{t^{1-2d\theta}}{214\times (3M^4)^d}\right),
   \end{align*}
   where 
   \begin{align*}
       A&:= C_1 (4sdh+2dr+d+r+dv) \log\left(h(3(d+s)+d)+2d+r\right)\\
       B&:=dt\left(2M^\prime t^\theta\right)^{dt}\log\left(dt\left(2M^\prime t^\theta\right)^{dt}\right)\\
       M^\prime&:=\frac{1}{\sqrt{c}}\tanh^{-1}(M\sqrt{c})
   \end{align*}
   By the statement of Theorem \ref{thm:main} and $(h,s,r)=(2,1,4)$ we get
   \begin{align*}
       \lim_{t\to\infty} \sup_{x_1^t} \mathcal{N}_1\left(\left(\frac{t^{-\theta}}{2M}\right)^d\frac{1}{20}, \pi_M\mathcal{T},x_1^t\right)\exp\left(-\frac{t^{1-2d\theta}}{214\times (3M^4)^d}\right)=0.
   \end{align*}
   Combining this with the Strong Law of Large Numbers, we get
   \begin{align*}
       \mathcal{E}_{\pi_{M_t}}(\pi_{M_t}f_{D}^{2,1,4})-\mathcal{E}_{\pi_{M_t},D}(\pi_{M_t}f_{D}^{2,1,4})=0
   \end{align*}
   holds almost surely, completing the proof of Lemma \ref{lem:5.2.4}. 
\end{proof}

\begin{rem}\label{rem:error_rate}
    \textbf{Sample Complexity of Transformers:} Lemma \ref{lem:5.2.4} (see Appendix) shows that for any finite truncation parameter $M$, the empirical error converges to the universal error at rate $\mathcal{O}(t^{-\theta})$ for some $\theta \in (0, 1/2d)$, with probability at least
\[
1 - C^\ast P \log(Q) \exp\left(-\frac{t^{1 - 2d\theta}}{214 \cdot (3M^4)^d}\right),
\]
where $P$ and $Q$ are as in Theorem \ref{thm:main}, and $C^\ast > 0$ is a constant. Thus, by the second Borel–Cantelli Lemma, $\mathcal{E}_D(f)$ converges almost surely to $\mathcal{E}(f_\rho)$ as $t \to \infty$. This implies that the rate of convergence of the error rate of $\mathcal{T}_{\mathcal{P},\mathcal{H}}^{2,1,4}$ is $\mathcal{O}(t^{-1/2d})$, where $t$ and $d$ are the number of input tokens and the input embedding dimensions, respectively. 

\end{rem}

Our final module to derive the universal consistency is the universal approximation properties of transformers, as discussed in \cite{approx_tran}. We can derive a similar approximation property for the generalized transformer architecture, which we prove in Lemma \ref{lem:univ_approx}.

\subsection{\textbf{Transformers as Universal Approximators}}

\begin{lemma}\label{lem:univ_approx}
    For any given $\epsilon>0$ and a function $g:\dd^{d\times t}\to\rr^{d\times t}$, there exists a HyT network $f\in\mathcal{T}_{\mathcal{P},\mathcal{H}}^{2,1,4}$ such that
    \begin{align}
        d_2(f,g):=\left[\int_{X}\|f(x)-g(x)\|^2_{L^2}dx\right]^{1/2}<\epsilon.
    \end{align}
\end{lemma}

\begin{proof}
    This lemma is a direct consequence of Theorem 3 \cite{yun2019transformers}. We define $h:=g\circ\exp_0^c$. Then $h$ is a continuous map from $\rr^{d\times t}\to\rr^{d\times t}$. Now $g$ is compactly supported and $\exp_0^c:\rr^{d\times t}\to\dd^{d\times t}$ is a global diffemorphism. So, $h$ is a continuous map from $\rr^{d\times t}\to\rr^{d\times t}$ with compact support. By Theorem 3 \cite{yun2019transformers}, we know that there exists a $t\in\mathcal{T}_\mathcal{P}^{2,1,4}$ such that 
    \begin{align}\label{eqn:univ_approx}
        \left[\int_{Y}\|h(y)-t(y)\|^2_{L^2}dy\right]^{1/2}<\frac{\epsilon}{c},
    \end{align}
    where $c=\sup_{x\in\mathcal{R}\subsetneq\dd^{d\times t}}\|\log_0^c(x)\|$. We have $c<\infty$, since $\mathcal{R}$ is a compact subset of $\dd^{d\times t}$ and $\log_0^c$ is a global diffeomorphism on $\mathcal{R}$. 
We will now construct the target function $f$ from $t$. Define $f:=t\circ\log_0^c$. Then $f\in\mathcal{T}_{\mathcal{P},\mathcal{H}}^{2,1,4}$. We also note that
\begin{align*}
    d_2(f,g) & = \left[\int_X\|f(x)-g(x)\|^2_{L^2}dx\right]^{1/2}\\
    &= \left[\int_X\|h\circ\log_0^c(x)-t\circ\log_0^c(x)\|^2_{L^2}dx\right]^{1/2}\\
    &= \left[\int\|(h-t)(\log_0^c(x))\|^2_{L^2}dx\right]^{1/2}\\
    &\leq c\left[\int_Y\|h(y)-t(y)\|^2_{L^2}dy\right]^{1/2}\\
    &\leq \epsilon,
\end{align*}
where the last inequality is followed from Equation \ref{eqn:univ_approx}. 
\end{proof}

\subsection*{Proof of Theorem \ref{thm:main}}
Since we know that $\mathbb{E}[(\log_0^c(y))^2]<\infty$, by Lemma \ref{lem:univ_approx}, for all $\epsilon>0$ there exists $f\in\mathcal{T}_{\mathcal{P},\mathcal{H}}^{2,1,4}$ such that  $\delta_2(f,f_\rho) < \epsilon$. 
    
    By the triangle inequality, we bound the following difference into a sum of $8$ terms,
    \begin{align*}
       & \ec(\pi_M(f_{D}^{2,1,4}))-\ec(f_\rho)\\
        \leq & \epsilon \ec\left(\pi_M(f_{D}^{2,1,4})\right)-(1+\epsilon)\ec(\pi_M(f_{D}^{2,1,4}))\\
        + & (1+\epsilon)\left(\ec_{\pi_M}(\pi_M(f_{D}^{2,1,4}))\right)-\ec_{\pi_M,D}(\pi_M(f_{D}^{2,1,4}))\\
        + & (1+\epsilon)\left(\ec_{\pi_{M},D}(\pi_M(f_{D}^{2,1,4}))-\ec_{\pi_M,D}(f_{D}^{2,1,4})\right)\\
        + & (1+\epsilon)(\ec_{\pi_{M},D}(f_{D}^{2,1,4}))-(1+\epsilon)^2(\ec_{D}(f_{D}^{2,1,4}))\\
        + & (1+\epsilon)^2\left(\ec_{D}(f_{D}^{2,1,4})-\ec_D(f)\right)\\
        + &(1+\epsilon)^2\left(\ec_{D}(f)-\ec(f)\right)\\
        + & (1+\epsilon)^2\left(\ec(f)-\ec(f_\rho)\right)+\left((1+\epsilon)^2-1\right)\ec(f_\rho)\\
        =: & \sum_{i=1}^8A_i,
    \end{align*}
where each of the $A_i$'s represents a term in the summation.  We will show that under the conditions of Theorem \ref{thm:main}, each of the $A_i\to0$. 
   We note that, for each $q,w,\epsilon>0$,
    \begin{equation} \label{ineq:9.1}
        (q+w)^2\leq(1+\epsilon)q^2+(1+1/\epsilon)w^2.
    \end{equation}  

We begin by showing that $A_1\to0$ as $t\to\infty$. 
    \begin{align*}
        A_1 & = \epsilon\ec\left(\pi_M(f_{D}^{2,1,4})\right)-(1+\epsilon)\ec(\pi_M(f_{D}^{2,1,4}))\\
        & = \int_{\mathcal{Z}}\|\pi_{M}(f_{D}^{2,1,4}(x))-(\log_0^c(y_M))\\
        &+(\log_0^c(y_M))-(\log_0^c(y))\|^2d\rho\\
        & - (1+\epsilon)\int_{\mathcal{Z}}\|\pi_M(f_{D}^{2,1,4})-(\log_0^c(y_M))\|^2d\rho\\
        & \leq (1+(1/\epsilon))\int_{\mathcal{Z}}|\log_0^c(y)-\log_0^c(y_M)|^2d\rho. \hspace{2ex} \text{[By (\ref{ineq:9.1})]}
    \end{align*}
    Since $\epsilon>0$ is arbitrary, by condition (i) in Theorem \ref{thm:main}, we have $A_1\to0$ as $t\to\infty$. Next, combining Lemma \ref{lem:univ_approx} and the conditions of Theorem \ref{thm:main}, we get $A_2\to 0 \hspace{1ex} \text{as} \hspace{1ex} t\to\infty$. The definition of the truncation operator results in 
   \begin{align*}
       A_3=&\frac{1}{t}\sum_{i=1}^t\|\pi_M(f_{D}^{2,1,4}(X_i))-(\log_0^c(Y_{i,M}))\|^2\\
       &-\frac{1}{t}\sum_{i=1}^t\|f_{D}^{2,1,4}(X_i)-(\log_0^c(Y_{i,M}))\|^2\leq 0.
   \end{align*}

   Using the Strong Law of Large Numbers and (\ref{ineq:9.1}), we deduce
   \begin{align*}
       A_4 \leq & (1+\epsilon)(1+1/\epsilon)\frac{1}{t}\sum_{i=1}^t\|\log_0^c(Y_i)-\log_0^c(Y_{i,M})\|^2\\
       &\to (1+\epsilon)(1+1/\epsilon)\int_{\mathcal{Z}}\|\log_0^c(y)-\log_0^c(y_M)\|^2d\rho
   \end{align*}
   as $t\to\infty$ almost surely. And the condition (i) in Theorem \ref{thm:main} implies $A_4\to 0$ as $t\to\infty$. Since $f_{D}^{2,1,4}$ is obtained by minimizing the empirical risk, we have,
   \begin{align*}
       A_5=&(1+\epsilon)^2\left(\frac{1}{t}\sum_{i=1}^t\|f_{D}^{2,1,4}(X_i)-\log_0^c(Y_i)|^2\right)\\
       &-(1+\epsilon)^2\left(\frac{1}{t}\sum_{i=1}^t|f(X_i)-\log_0^c(Y_i)|^2\right)\leq 0.
   \end{align*}
   Once again, the strong law of large numbers directly indicates $A_6\to0$ almost surely as $t\to\infty$. We can write $A_7$ as $A_7=(1+\epsilon)^2\|f-f\rho\|^2_{L^2_{\rho_\xx}}$, which goes to $0$ by Lemma \ref{lem:univ_approx}. Finally, for $A_8$, we have
   \begin{align*}
       A_8&\leq ((1+\epsilon)^2-1)\int_{\mathcal{Z}}\|f_\rho(x)-\log_0^c(y)\|^2d\rho \\
       &= \epsilon(\epsilon+2)\int_{\mathcal{Z}}\|f_\rho(x)-\log_0^c(y)\|^2d\rho,
   \end{align*}
which definitely goes to $0$ as $\epsilon>0$ is arbitrary and the integral is bounded. This completes the proof of Theorem \ref{thm:main}.

\section{Experiments \& Results}\label{sec:7}

We will discuss the results obtained from our simulation on real-world question answering datasets such as Squad, BoolQ, etc, on a BERT-like only-encoder based Transformer Model. TO validate our theoretical claims, we will run our simulation on the same model with four different curvatures: curvature $0$ (Euclidean) and on Poincar\'{e} Balls with curvatures $0.0001, 1.0$, and $100$ respectively. The Python-based implementation is available at \href{https://github.com/sagarghosh1729/Hyperbolic_BERT.git}{https://github.com/sagarghosh1729/Hyperbolic\_BERT.git}.

\subsection{\textbf{Model Description}} 
Our model is conceptually grounded in the architecture of the standard Euclidean Transformer and its hyperbolic generalizations. Notably, the Euclidean variant emerges as a computationally efficient degenerate case corresponding to zero curvature (\(c = 0\)) within the broader family of hyperbolic models. This structural relationship facilitates the transfer of empirical observations from the Euclidean instance to its hyperbolic counterparts. Accordingly, we utilize a minimal encoder-only architecture comprising a single block, configured with the following hyperparameters: embedding dimension \(d = 128\), number of attention heads \(h = 2\), head size \(s = 1\), and feed-forward hidden dimension \(r = 4\). These choices are designed to align with the theoretical assumptions underpinning our analysis.

For practical implementation, we adopt the vocabulary, pretrained WordPiece tokenizer, and sinusoidal positional encodings from a standard BERT model. The network is trained as a regressor using the Mean Squared Error (MSE) loss function and optimized with the Adam algorithm, employing a learning rate of \(0.001\) and weight decay of \(0.01\). For each dataset and fixed curvature setting, we perform a single training run and report the Test Root Mean Squared Error (Test RMSE) as a function of training iterations.

\subsection{\textbf{Dataset Description}}
\begin{enumerate}
    \item \textbf{Squad}
    \begin{itemize}
        \item \textit{No of Samples:} $\sim 100k$
        \item \textit{Type:} Extractive QA type, answers span on a continuous block of text in the context. 
        \item \textit{Use Cases:} Popular benchmarks for training and evaluating extractive QA models like BERT, a regression setup.
    \end{itemize}
    \item \textbf{Bool QA}
     \begin{itemize}
        \item \textit{No of Samples:} $\sim 12k$
        \item \textit{Type:} Binary question answer type
        \item \textit{Use Cases:} Simple question-answer with a classification setup.
    \end{itemize}
    \item \textbf{Tweet QA}
     \begin{itemize}
        \item \textit{No of Samples:} $\sim 13k$
        \item \textit{Type:} Tweets
        \item \textit{Use Cases:} Presence of noise, more towards real-world-like data, a regression setup
    \end{itemize}
\end{enumerate}

\subsection{\textbf{Dataset Specific Model Descriptions}}
The Table \ref{tab:hyper_details_transformer} refers to the simulation-specific hyperparameters and other details. 

\begin{table*}[!ht]
\centering
\caption{The complete details of hyperparameters for three real-world datasets are presented to reproduce the results. }
\label{tab:hyper_details_transformer}
\resizebox{0.8\textwidth}{!}{
\begin{tabular}{lccccc}
\toprule
Hyperparameter Details  &  Squad QA Dataset &   Bool QA Dataset  &  Tweet QA Dataset  \\
\midrule
No of Encoders & $1$ & $1$ & $1$  & \\
Input Embedding Dimension & $128$ & $128$ & $128$ \\
Noise  & Yes & Yes & Yes  \\
Learning Rate & $0.001$ & $0.001$ & $0.001$  \\
Weight decay & $0.0005$ & $0.0005$ & $0.0005$ \\
Train/test split & $0.80$ & $0.80$ & $0.80$ \\
No of samples & $\sim 100k$ & $\sim 12k$ & $\sim 13k$ \\
Batch Size & $16$ & $16$ & $16$ \\
Max Sequence length & $512$ & $512$ & $512$ \\
$(h,s,r)$ & $(2,1,4)$ & $(2,1,4)$ & $(2,1,4)$ \\
Optimizer &Adam & Adam & Adam  \\
No of Input Tokens & $\sim12M$& $\sim0.4M$& $\sim 0.4M$ \\
\bottomrule
\end{tabular}}
\end{table*}


\begin{figure}
    \centering
    \includegraphics[width=1.0\linewidth]{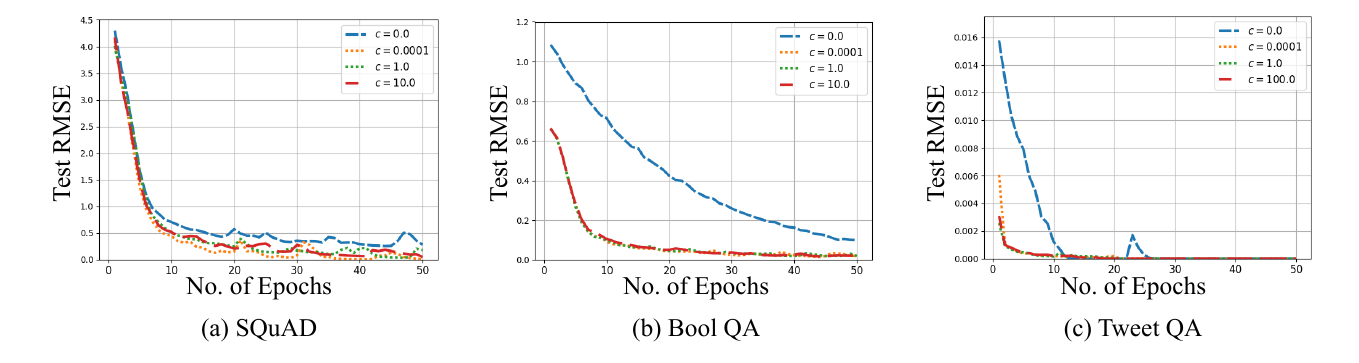}
    \caption{Test RMSE Loss over the number of training iterations of the three Question-Answering datasets shown above}
    \label{fig:bert_sim_1}
\end{figure}
\subsection{\textbf{Discussion on Simulations}}
We now present and examine the empirical results obtained from our experiments across the five datasets previously introduced. The corresponding figure provides a visual aid to support the ensuing analysis. As the sequence-to-sequence translation task has been cast as a large-scale regression problem, guided by the principle of minimizing the \(L^2\) generalization error, our primary evaluation metric is the Test Root Mean Squared Error (Test RMSE), plotted as a function of training iterations.

For each dataset, we instantiated the BERT-inspired encoder architecture with the hyperparameter configuration \((h, s, r) = (2, 1, 4)\), running the model for 50 training iterations while evaluating four distinct curvature settings: \(c = 0\), \(0.0001\), \(1.0\), and \(10.0\). The resulting Test RMSE trajectories consistently highlight the advantages conferred by the hyperbolic generalization. Interestingly, a modest decline in performance is observed at higher curvature values—an anticipated phenomenon attributable to the geometric contraction of the Poincaré disk as the curvature parameter \(c \to \infty\).

\subsection{\textbf{Analyzing the Sample Complexity of Transformers}}

In this section, we empirically substantiate the claim articulated in Remark \ref{rem:error_rate} concerning the sample complexity behavior on the TweetQA dataset. To this end, we repeated the same experiment ten times at the 20\textsuperscript{th} training epoch, recording the mean Test Root Mean Squared Error (Test RMSE) and its associated standard deviation at token counts in increments of 20,000. The results are visually summarized in Figure \ref{fig:tweet_error_rate}.

Notably, the Test RMSE trajectories, along with their standard deviation bands, consistently remain beneath the reference curve \(y = t^{-1/257}\) (up to a constant scaling factor) as the number of training tokens increases. This empirical observation lends strong support to the theoretical claim made in Remark \ref{rem:error_rate}.

\begin{figure}[!ht]
    \centering
    \includegraphics[width=1.0\linewidth]{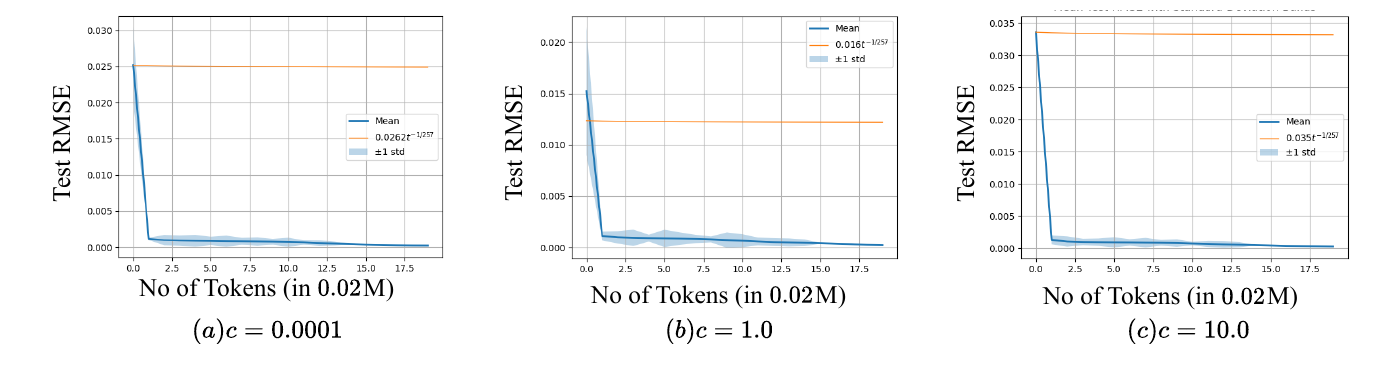}
    \caption{Test RMSE Loss over the number of training tokens (in $0.02$ Millions) of the Tweet QA dataset over the course of three different curvatures}
    \label{fig:tweet_error_rate}
\end{figure}

\section{\textbf{Conclusion and Future Works}}\label{sec:8}

\subsection*{\textbf{Conclusion}}
We have established that both vanilla Transformers and their hyperbolic counterparts exhibit universal consistency under appropriate regularity assumptions. This ensures that, given sufficient training data and sequence-to-sequence modeling, no alternative model can asymptotically outperform them in terms of generalization error.

Our results thus provide a theoretical foundation complementing the empirical success of Transformers in domains such as NLP~\cite{vaswani2017attention}, vision~\cite{dosovitskiy2020image}, and reinforcement learning~\cite{chen2021decision}. Specifically, universal consistency guarantees that, with increasing data and proper scaling, Transformers converge to the optimal predictor—a cornerstone of statistical learning theory~\cite{gyor}.

Further, our sample complexity analysis quantifies the convergence rate of empirical to generalization error as $\mathcal{O}(t^{-1/2d})$, where $d$ denotes the embedding dimension and $t$ the number of tokens. This provides practical insight into data requirements for near-optimal performance, aligning with observed scaling laws.

By extending our framework to hyperbolic settings, we generalize consistency guarantees to non-Euclidean spaces, enabling effective modeling of hierarchical and graph-structured data~\cite{kaplan2020rewon}.

\subsection*{\textbf{Future Works}}
While our analysis assumes idealized settings and i.i.d. data, it lays the groundwork for exploring consistency under more realistic scenarios, including distribution shifts, adversarial perturbations, and structured outputs. Future directions include deriving tighter convergence bounds, extending to non-i.i.d. regimes, and characterizing generalization in downstream tasks.

In essence, our findings affirm that attention-based architectures are not only empirically effective but also theoretically grounded as asymptotically optimal learners.

\bibliographystyle{IEEEtran}
\bibliography{ref}

\clearpage

\end{document}